\documentclass[english,envcountsame]{llncs}
\usepackage[T1]{fontenc}
\usepackage[latin9]{inputenc}
\usepackage{array}
\usepackage{float}
\usepackage{amsmath}
\usepackage{amssymb}
\usepackage{stmaryrd}
\usepackage{setspace}
\makeatletter

\providecommand{\tabularnewline}{\\}

\usepackage{xcolor}
\definecolor{darkred}{rgb}{0.8,0,0} 
\usepackage{hyperref}
\hypersetup{final, colorlinks=true, linkcolor=black, urlcolor=black, citecolor=black, linkbordercolor=black, pdfborderstyle={/S/U/W 1}}
\usepackage{setspace}
\usepackage[absolute,overlay]{textpos} 
\usepackage{listings}
\usepackage{cite}
\usepackage{tikz}
\usetikzlibrary{automata,arrows,shapes,decorations,topaths,trees,backgrounds,shadows,positioning,fit,calc}

\makeatletter
\g@addto@macro \normalsize {%
 \setlength\abovedisplayskip{4pt plus 2pt minus 2pt}%
 \setlength\belowdisplayskip{4pt plus 2pt minus 2pt}%
}
\makeatother

\DeclareRobustCommand{\VAN}[2]{#2}

\usepackage{nameref}
\makeatletter
\let\orgdescriptionlabel\descriptionlabel
\renewcommand*{\descriptionlabel}[1]{%
  \let\orglabel\label
  \let\label\@gobble
  \phantomsection
  \edef\@currentlabel{#1}%
  \let\label\orglabel
  \orgdescriptionlabel{#1}%
}
\makeatother

\makeatletter
\newcommand*{\textlabel}[2]{%
  \edef\@currentlabel{#1}
  \phantomsection
  #1\label{#2}
}
\makeatother

\AtBeginDocument{
  
}

\makeatother

\usepackage{babel}
\begin{document}
\setstretch{1}
\title{Awareness Logic: A Kripke-based Rendition of the Heifetz-Meier-Schipper Model}
\author{Gaia Belardinelli and Rasmus K. Rendsvig \institute{Center for Information and Bubble Studies, University of Copenhagen \\ \email{\{belardinelli,rasmus\}@hum.ku.dk}}}
\maketitle
\begin{abstract}
Heifetz, Meier \& Schipper (HMS) present a lattice model of awareness.
The HMS model is syntax-free, which precludes the simple option to
rely on formal language to induce lattices, and represents uncertainty
and unawareness with one entangled construct, making it difficult
to assess the properties of either. Here, we present a model based
on a lattice of Kripke models, induced by atom subset inclusion, in
which uncertainty and unawareness are separate. We show the models
to be equivalent by defining transformations between them which preserve
formula satisfaction, and obtain completeness through our and HMS'
results.
\end{abstract}

\section{Introduction}

Awareness has been studied with vigor in logic and game theory since
its first formal treatment by Halpern and Fagin in \cite{HalpernFagin88}.
In these fields, awareness is added as a complement to uncertainty
in models for knowledge and rational interaction. In short, where
uncertainty concerns an agent's ability to distinguish possible states
of the world based on its available information, awareness concerns
the agent's ability to even contemplate aspects of a state, where
such inability stems from the \emph{unawareness }of the concepts that
constitute said aspects. Thereby, models that include awareness avoid
problems of logical omniscience (at least partially) and allows modeling
game theoretic scenarios where the possibility of some action may
come as an utter surprise.

To model awareness, the seminal \cite{HalpernFagin88} introduces
the Logic of General Awareness (LGA), taking a syntax-based approach:
an agent $a$'s awareness in state $w$ is given by an \emph{awareness
function }assigning $(a,w)$ a set of formulas. This approach has
since been inherited by a multitude of models.

In contrast, Heifetz, Meier and Schipper (HMS) construct a syntax-free
framework \cite{HMS2006}, which is the main topic of this paper.
In their \emph{unawareness frames}, both ``atomic'' and epistemic
events are defined without any appeal to atomic propositions or other
syntax.

The backbone of an unawareness frame is a complete lattice of state-spaces
$(\mathcal{S},\preceq)$, with the intuition that the higher a space
is, the richer the ``vocabulary'' it has to describe its states.
Since the approach is syntax-free, this intuition is not modeled using
a formal language. It is represented using $\preceq$ and a family
of maps $r_{S}^{S'}$ which projects state-space $S'$ down to $S$,
with $r_{S}^{S'}(s)$ interpreted as the representation of $s$ in
the more limited vocabulary available in $S$. Uncertainty and unawareness
are represented \emph{jointly} by a \emph{possibility correspondence
$\Pi_{a}$} for each $a\in Ag$, which maps a state weakly downwards
to the set of states the agent considers possible. If the mapped-to
space is strictly less expressive, this represents that the agent
does not have full awareness of the mapped-from state.

That HMS keep their model syntax-free is motivated in part by its
applicability among economists \cite[p. 79]{HMS2006}. We think their
lattice-based conceptualization of awareness is both elegant, interesting
and intuitive\textemdash but we also find its formalization cumbersome.
Exactly the choice to go fully syntax-free robs the model of the option
to rely on formal language to induce lattices and to specify events,
resulting in constructions which we find less than very easy to follow.
This may, of course, be an artifact of us being accustomed to non-syntax-free
models used widely in epistemic logic.

Another artifact of our familiarity with epistemic logic models is
that we find HMS' joint definition of uncertainty and unawareness
difficult to relate to other formalizations of knowledge. When HMS
propose properties of their $\Pi_{a}$ maps, it is not clear to us
which aspects concern knowledge and which concern awareness. They
merge two dimensions which, to us, would be clearer if left separated.\footnote{As a reviewer points out, then HMS take \emph{explicit }knowledge
as foundational, and derive awareness from it. This makes the one-dimensional
representation justified, if not even desirable. In contrast, epistemic
logic models are standardly interpreted as taking \emph{implicit }knowledge
as foundational. We think along the second line, and add awareness
as a second dimension. We are not taking a stand on whether one interpretation
is superior, but provide results to move between them.} 

\bigskip{}
With these two motivations, this paper proposes a non-syntax-free,
Kripke model-based rendition of the HMS model. Roughly, we suggest
to start from a Kripke model $\mathtt{K}$ for a set of atoms $At$,
spawn a lattice containing restrictions of $\mathtt{K}$ to subsets
of $At$, and finally add maps $\pi_{a}$ on the lattice that take
a world to a copy of itself in a restricted model. This keeps the
epistemic and awareness dimensions separate: accessibility relations
$R_{a}$ of $\mathtt{K}$ encode epistemics while maps $\pi_{a}$
encode awareness. We show that under three assumptions on $\pi_{a}$
and when each $R_{a}$ is an equivalence relation, the result is equivalent
to the HMS model, in the sense that the two satisfy the same formulas
of the language of knowledge and awareness, defined below.

Defining an equivalent model, we do not aim to generalize that of
HMS, but we do include an additional perspective. \cite{HMS2006,Schipper2014}
argue that the HMS model allows agents to reason about their unawareness,
as possibility correspondences $\Pi_{a}$ provide them a subjective
perspective, while LGA-based approaches only present an outside perspective,
as the full model must be taken into account when assigning knowledge
and awareness.\footnote{\cite{HalpernRego2008} argues that this boils down to a difference
in philosophical interpretation.} Oppositely, Halpern and Rêgo \cite{HalpernRego2008} point out that
the HMS model includes no objective state, and so no outside perspective.
The present model has both: the starting Kripke model provides an
outsider perspective on agents' knowledge, while the submodel obtained
by following $\pi_{a}$ presents the subjective perspective. We remark
further on this below.

The paper progresses as follows. Sections \ref{sec:HMS} and \ref{sec:Kripke-Lattice-Models}
present respectively the HMS model and our rendition. Sections \ref{sec:Moving}
and \ref{sec:Language} contain our main technical results: Section
\ref{sec:Moving} introduces transformations between the two models
classes, while Section \ref{sec:Language} shows that they preserve
formula satisfaction. Section \ref{sec:Logic} presents a logic due
to HMS \cite{HMS2008}, and shows, as a corollary to our results,
that it is complete with respect to our rendition. Section \ref{sec:Concluding-Remarks}
holds concluding remarks.\smallskip{}

Throughout the paper, we assume that $Ag$ is a finite, non-empty
set of agents, and that $At$ is a countable, non-empty set of atoms.

\section{\label{sec:HMS}The HMS Model}

This section presents HMS unawareness frames \cite{HMS2006}, their
syntax-free notions of knowledge and awareness, and their augmentation
with HMS valuations, producing HMS models \cite{HMS2008}. For context,
the HMS model is a multi-agent generalization of the Modica-Rustichini
model \cite{ModicaRustichini1999} which is equivalent to Halpern's
model in \cite{halpern2001alternative}, generalized by Halpern and
Rêgo to multiple agents \cite{HalpernRego2008}, resulting in a model
equivalent to the HMS model, cf. \cite{HMS2008}. See \cite{Schipper2014}
for an extensive~review.

The following definition introduces the basic structure underlying
the HMS model, as well as the properties of the $\Pi_{a}$ map that
controls the to-be-defined notions of knowledge and awareness. The
properties are described after Definition~\ref{def:unawareness-frame}.
Following Definition~\ref{def:HMS model} of HMS models, Figure~\ref{fig:HMSmodel}
illustrates a full HMS model, including its unawareness frame.
\begin{definition}
\label{def:unawareness-frame}An \textbf{unawareness frame} is a tuple
$\mathsf{F}=(\mathcal{S},\preceq,\mathcal{R},\Pi)$ where\smallskip{}

\noindent $(\mathcal{S},\preceq)$ is a complete lattice with $\mathcal{S}=\{S,S',...\}$
a set of disjoint, non-empty \textbf{state-spaces} $S=\{s,s',...\}$
s.t. $S\preceq S'$ implies $|S|\leq|S'|$. Let $\Omega_{\mathsf{F}}:=\bigcup_{S\in\mathcal{S}}S$
be the disjoint union of state-spaces~in~$\mathcal{S}$. For $X\subseteq\Omega_{\mathsf{F}}$,
let $S(X)$ be the state-space containing $X$, if such exists (else
$S(X)$ is undefined). Let $S(s)$ be $S(\{s\})$.\smallskip{}

\noindent $\mathcal{R}=\{r_{S}^{S'}\colon S,S'\in\mathcal{S},S\preceq S'\}$
is a family of \textbf{projections} $r_{S}^{S'}:S'\rightarrow S$.
Each $r_{S}^{S'}$ is surjective, $r_{S}^{S}$ is $Id$, and $S\preceq S'\preceq S''$
implies commutativity: $r_{S}^{S''}=r_{S}^{S'}\circ r_{S'}^{S''}$.
Denote $r_{S}^{T}(w)$ also by $w_{S}$.

$D^{\uparrow}=\bigcup_{S'\succeq S}(r_{S}^{S'})^{-1}(D)$ is the \textbf{upwards
closure }of $D\subseteq S\in\mathcal{S}$.\footnote{\textit{\emph{To avoid confusion, note that for $d\in S$, $(r_{S}^{S'})^{-1}(d)=\{s'\in S':r_{S}^{S'}(s')=d\}$
and for $D\subseteq S$, $(r_{S}^{S'})^{-1}(D)=\bigcup_{d\in D}(r_{S}^{S'})^{-1}(d)$.}}}\smallskip{}

\noindent $\Pi$ assigns each $a\in Ag$ a \textbf{possibility correspondence
}$\Pi_{a}:\Omega_{\mathsf{F}}\rightarrow2^{\Omega_{\mathsf{F}}}$
satisfying\vspace{-6pt}
\begin{description}
\item [{\label{HMS1:confinement}Conf}] (\textbf{Confinement})\quad{}
If $w\in S'$, then $\Pi_{a}(w)\subseteq S$ for some $S\preceq S'$.
\item [{Gref\label{HMS2:G.Ref}}] (\textbf{Generalized~Reflexivity}) $w\in\left(\Pi_{a}(w)\right)^{\uparrow}$
for every $w\in\Omega_{\mathsf{F}}$.
\item [{\label{HMS3:stationarity}Stat}] (\textbf{Stationarity}) $w'\in\Pi_{a}(w)$
implies $\Pi_{a}(w')=\Pi_{a}(w)$.
\item [{\label{HMS4:PPI}PPI}] (\textbf{Projections~Preserve~Ignorance})
If $w\in S'$ and $S\preceq S'$, then\linebreak{}
 $(\Pi_{a}(w))^{\uparrow}\subseteq(\Pi_{a}(r_{S}^{S'}(w)))^{\uparrow}$.
\item [{\label{HMS5:PPK}PPK}] (\textbf{Projections~Preserve~Knowledge})
If $S\preceq S'\preceq S''$, $w\in S''$ and $\Pi_{a}(w)\subseteq S'$,
then $r_{S}^{S'}(\Pi_{a}(w))=\Pi_{a}(r_{S}^{S''}(w))$.\vspace{-6pt}
\end{description}
Jointly call these five properties of $\Pi_{a}$ the \textbf{\textbf{\textlabel{HMS properties}{text:HMSproperties}}}\label{text:HMSproperties}.
\end{definition}

\ref{HMS1:confinement} ensures that agents only consider possibilities
within one fixed ``vocabulary''; \ref{HMS2:G.Ref} induces factivity
of knowledge and \ref{HMS3:stationarity} yields introspection for
knowledge and awareness. \ref{HMS4:PPI} entails that at down-projected
states, agents neither ``miraculously'' know or become aware of
something new, while \ref{HMS5:PPK} implies that at down-projected
states, the agent can still ``recall'' all events she knew before,
if they are still expressible. Jointly \ref{HMS4:PPI} and \ref{HMS5:PPK}
imply that agents preserve awareness of all events at down-projected
states, if they are still expressible.
\begin{remark}
\label{rem:No-Objective-State}Unawareness frames include no objective
perspective, as agents do not\textemdash unless they are fully aware\textemdash have
a range of uncertainty defined for the maximal state-space. Taking
the maximal state-space to contain a designated `actual world' and
as providing a full and objective description of states, one can still
not evaluate agents ``true'' uncertainty/implicit knowledge. See
e.g. Figure~\ref{fig:HMSmodel} below: In $(\neg i,\ell)$, the dashed
agent's ``true'' uncertainty about $\ell$ is not determined.
\end{remark}

\subsection{Syntax-Free Unawareness}

Unawareness frames provide sufficient structure to define syntax-free
notions of knowledge and awareness. These are defined directly as
events on $\Omega_{\mathsf{F}}$.
\begin{definition}
Let $\mathsf{F}=(\mathcal{S},\preceq,\mathcal{R},\Pi)$ be an unawareness
frame. An \textbf{event} in $\mathsf{F}$ is any pair $(D^{\uparrow},S)$
with $D\subseteq S\in\mathcal{S}$ with $S$ also denoted $S(D^{\uparrow})$.
Let $\Sigma_{\mathsf{F}}$ be the set of events of $\mathsf{F}$.

The \textbf{negation} of the event $(D^{\uparrow},S)$ is $\neg(D^{\uparrow},S)=((S\backslash D)^{\uparrow},S)$.

The \textbf{conjunction} of events $\{(D_{i}^{\uparrow},S_{i})\}_{i\in I}$
is $((\bigcap_{i\in I}D_{i}^{\uparrow}),\sup_{i\in I}S_{i})$.

The events that $a$ \textbf{knows} event $(D^{\uparrow},S)$ and
where $a$ is \textbf{aware} of it are \vspace{-8pt}

{\small{}
\begin{align*}
\boldsymbol{K}_{a}((D^{\uparrow},S)) & =\begin{cases}
(\{w\in\Omega_{\mathsf{F}}\colon\Pi_{a}(w)\subseteq D^{\uparrow}\},S(D)) & \text{\,\,\,\,\,\,\,\,\,\,\,if }\exists w\in\Omega_{\mathsf{F}}.\Pi_{a}(w)\subseteq D^{\uparrow}\\
(\emptyset,S(D)) & \text{\,\,\,\,\,\,\,\,\,\,\,else}
\end{cases}\\
\boldsymbol{A}_{a}((D^{\uparrow},S)) & =\begin{cases}
(\{w\in\Omega_{\mathcal{\mathsf{F}}}\colon\Pi_{a}(w)\subseteq S(D^{\uparrow})^{\uparrow}\},S(D)) & \text{if }\exists w\in\Omega_{\mathsf{F}}.\Pi_{a}(w)\subseteq S(D^{\uparrow})^{\uparrow}\\
(\emptyset,S(D)) & \text{else}
\end{cases}
\end{align*}
}{\small\par}
\end{definition}

\noindent Negation, conjunction, knowledge and awareness events are
well-defined \cite{HMS2006,Schipper2014}. To illustrate the definitions,
some intuitions behind them: $i)$ an event modeled as a pair $(D^{\uparrow},S)$
captures that $a)$ if the event is expressible in $S$, then it is
also expressible in any $S'\succeq S$, hence $D^{\uparrow}$ is the
set of all states where the event is expressible and occurs, and $b)$
the event is expressible in the``vocabulary'' of $S$, but not the
``vocabulary'' of lower state-spaces: $D\subseteq S$ are the states
with the lowest ``vocabulary'' where the event is expressible and
occurs. \cite{Schipper2014} remarks that for $(D^{\uparrow},S)$,
if $D\ne\emptyset$, then $S$ is uniquely determined by $D^{\uparrow}$.
$ii)$ Events are given a non-binary understanding: an event $(D^{\uparrow},S)$
and it's negation does not partition $\Omega_{\mathsf{F}}$, as $s\in S'\prec S$
is in neither, but they do partition every $S''\succeq S$. $iii)$
Conjunction defined using supremum captures that the state-space required
to express the conjunction of two events is the least expressive state-space
that can express both events. $iv)$ Knowledge events are essentially
defined as in Aumann structures/state-space models: the agent knows
an event if its ``information cell'' is a subset of the event's
states. $v)$ Awareness events captures that ``\emph{an agent is
aware of an event if she considers possible states in which this event
is \textquotedblleft expressible\textquotedblright }.''\cite[p. 97]{Schipper2014}

\subsection{HMS Models}

Though unawareness frames provide a syntax-free framework adequate
for defining awareness, HMS \cite{HMS2008} use them as a semantics
for a formal language in order to identify their logic. The language
and logic are topics of Sections \ref{sec:Language} and \ref{sec:Logic}.

Instead, the models we will later define are not syntax-free. As Kripke
models, they include a valuation of atomic propositions. Therefore,
they do not correspond to unawareness frames directly, but to the
models that result by augmenting such frames with valuations. To compare
the two model classes, we define such valuations here, postponing
HMS syntax and semantics to Section~\ref{sec:Language}. Figure~\ref{fig:HMSmodel}
illustrates an HMS model, using an example inspired by \cite[p. 87]{HMS2006}
\begin{definition}
\label{def:HMS model} Let $\mathsf{F}=(\mathcal{S},\preceq,\mathcal{R},\Pi)$
be an unawareness frame with events $\Sigma_{\mathsf{F}}$. An \textbf{HMS
valuation }for $At$ and $\mathsf{F}$ is a map $V_{\mathsf{M}}:At\rightarrow\Sigma_{\mathsf{F}}$,
assigning to every atom from $At$ an event in $\mathsf{F}$. An \textbf{HMS
model} is an unawareness frame augmented with an HMS valuation, denoted
$\mathsf{M}=(\mathcal{S},\preceq,\mathcal{R},\Pi,V_{\mathsf{M}})$.\vspace{-28pt}
\end{definition}

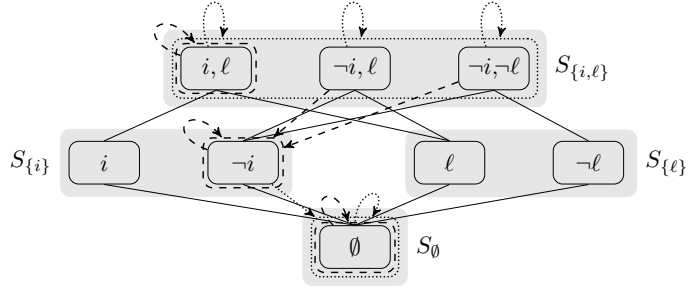
\begin{figure}[H]
\begin{centering}
\resizebox{0.8\textwidth}{!}{
	\pgfdeclarelayer{background}
	\pgfdeclarelayer{foreground}
	\pgfsetlayers{background,main,foreground}

	\begin{tikzpicture}
	\tikzset{world/.style={rectangle, draw=black, rounded corners, text width=23pt, minimum height=18pt, 
			text centered},
		modal/.style={>=stealth',shorten >=1pt,shorten <=1pt,auto,node distance=1cm},
		state/.style={circle,draw,inner sep=0.5mm,fill=black}, 
		PCon/.style ={rectangle,draw=black, rounded corners, semithick, densely dotted, text width=30pt, minimum height=18pt, text centered},
		PCbn/.style ={draw=black, rounded corners, semithick, dashed, text width=34pt, minimum height=21pt}, 
		PCo/.style={->,densely dotted, semithick, >=stealth'},
		EPCo/.style={-,densely dotted, semithick},
		EPCb/.style={-,dashed, semithick},
		PCb/.style={->,dashed, semithick, >=stealth'}, 
		proj/.style={-, line width=0.001pt},
		block/.style ={rectangle, fill=black!10,draw=black!10, rounded corners, text width=5em, minimum height=0.3cm},
		reflexive above/.style={->,out=60,in=100,looseness=8},
		reflexive left/.style={->,out=-170,in=-190,looseness=8},
	}
	
	\begin{pgfonlayer}{foreground}
	\node[world] (a) {$i,\ell$};

	\node[world, right= of a] (b) {$\neg i,\ell$};
	\node[world, right= of b] (c) {$  \neg i,\!\!\neg \ell$};

	\node[world, below =of a, xshift=4mm, yshift=2.5mm] (b') {$\neg i$};
	\node[world, left=of b'] (a') {$i$};

	\node[world, right=of b', xshift=10mm] (a'') {$\ell$};
	\node[world, right=of a''] (b'') {$\neg \ell$};

	\node[world, below=of b, yshift=-10mm] (a''') {$\emptyset$};
	\end{pgfonlayer}
	
	\node [PCbn](g)[text width=27pt, minimum height=21pt] {}; 	
	
	\node [PCbn, above=-7.05mm of b'] (g') [text width=27pt, minimum height=21pt] {}; 	
	\node [PCon,  fit= (a)(b)(c)] (g''){}; 	
	\node [PCon, fit=(a''')] (g''') {};

	\draw [PCb] (a) to [out=-200,in=130,looseness=7] (g)  node {$ $};
	
	\node[above=-3pt of a, xshift=2pt] (anchor1) {};
	\draw[PCo] (a) to [out=110,in=80,looseness=8] (anchor1);
	
	\node[above=-3pt of b, xshift=2pt] (anchor1) {};
	\draw[PCo] (b) to [out=110,in=80,looseness=8] (anchor1);
	
	\node[above=-3pt of c, xshift=2pt] (anchor1) {};
	\draw[PCo] (c) to [out=110,in=80,looseness=8] (anchor1);
	
	\draw[PCb] (b') to [out=-200,in=130,looseness=6.5] (g');
	
	\node [PCbn, above=-7.1mm of a'''](h')[text width=27pt, minimum height=21pt] {};

	\draw[PCb] (a''') to [out=135,in=105,looseness=7] (h');
	\draw[PCo] (a''') to [out=90,in=60,looseness=5] (g''');
	
	\draw [PCb] (b) to (g')  node {$ $};
	\draw [PCb] (c) to (g')  node {$ $};
	\draw [PCo] (b') to (g''')  node {$ $};

	\node [ above=-7.1mm of b''](h'')[text width=27pt, minimum height=21pt] {}; 				
	\begin{pgfonlayer}{background}
	\node[block, fit=(g'')] (block) [label= right:{$S_{\{i,\ell\}}$}]{ };
	\node[block, fit=(g')(a')] (block) [label=left: {$S_{\{i\}}$}]{ };
	
	\node[block, fit=(a'')(h'')] (block) [label= right:{$S_{\{\ell\}}$}]{ };
	
	\node[block, fit=(a''')(g''')] (block) [label= right:{$S_{\emptyset}$}]{ };
	\end{pgfonlayer}       
	
	
	\draw [-] (a.south) -- (a''.north)[line width=0.001pt];
	\draw [-] (b.south) -- (a''.north)[line width=0.001pt];
	\draw [-] (c.south) -- (b''.north)[line width=0.001pt];
	
	\draw [-] (a.south) -- (a'.north)[line width=0.001pt];
	\draw [-] (b.south) -- (b'.north)[line width=0.001pt];
	\draw [-] (c.south) -- (b'.north)[line width=0.001pt];
	
	\draw [-] (a'.south) -- (a'''.north)[line width=0.001pt];
	\draw [-] (b'.south) -- (a'''.north)[line width=0.001pt];
	
	\draw [-] (a''.south) -- (a'''.north)[line width=0.001pt];
	\draw [-] (b''.south) -- (a'''.north)[line width=0.001pt];

	\end{tikzpicture}
}


\vspace{-8pt}
\par\end{centering}
\caption{\label{fig:HMSmodel}An HMS model with four state-spaces (gray rectangles),
ordered spatially as a lattice. States (smallest rectangles) are labeled
with their true literals, over the set $At=\{i,\ell\}$. Thin lines
between states show projections. There are two possibility correspondences
(dashed and dotted): arrow-to-rectangle shows a mapping from state
to set (information cell). Omitted arrows go to $S_{\emptyset}$ and
are irrelevant to the story.\textbf{$\protect\phantom{aaa}$ $\protect\phantom{aaa}$Story:}
Buyer (dashed) and Owner (dotted) consider trading a firm, the price
influenced by whether $i$ (a value-raising innovation) and $\ell$
(a value-lowering lawsuit) occurs. Assume both occur and take $(i,\text{\ensuremath{\ell}})$
as actual. Then Buyer has full information, while Owner has factual
uncertainty and uncertainty about Buyer's awareness and higher-order
information, ultimately considering it possible that Buyer holds Owner
fully unaware. \emph{In detail:} Buyer's $(i,\ell)$ information cell
has both $i$ and $\ell$ defined (and is also singleton), so Buyer
is aware of them (and also knows everything). Owner is also aware
of $i$ and $\ell$, but their $(i,\ell)$ information cell contains
also $\neg i$ and $\neg\ell$ states, so Owner knows neither. Owner
is also uncertain about Buyer's information: Owner knows that either
Buyer knows $i$ and $\ell$ (cf. Buyer's $(i,\ell)$ information
cell), or Buyer knows $\neg i$, but is unaware of $\ell$ (cf. the
dashed arrows from $\neg i$ states to the less expressive state space
$S_{\{i\}}$) and then only holds it possible that Owner is unaware
of both $i$ and $\ell$ (cf. the dotted map to $S_{\emptyset}$).
See also Remark \ref{rem:HMS-redundant-states} concerning~$S_{\{\ell\}}$.}
\end{figure}

\begin{remark}
\noindent \label{rem:HMS-valuations-only}HMS valuations only partially
respect the intuitive interpretation of state-spaces lattices, where
$S\preceq S'$ represents that $S'$ is at least as expressive as
$S$. If $S\preceq S'$, then $p\in At$ having defined truth value
at $S$ entails that it has defined truth value at $S'$, but if $S$
is strictly less expressive than $S'$, then this does not entail
that there is some atom $q$ with defined truth value in $S'$, but
undefined truth value in $S$. Hence, there can exist two spaces defined
for the same set of atoms, but where one is still ``strictly more
expressive'' \mbox{than the other.}
\end{remark}

\begin{remark}
\label{rem:HMS-redundant-states} Concerning Figure~\ref{fig:HMSmodel},
then the state-space $S_{\{\ell\}}$ is, in a sense, redundant: its
presence does not affect the knowledge or awareness of agents in the
state $(i,\ell)$, and it presence is not required by definition.
This stands in contrast with the corresponding Kripke lattice model
in Figure~\ref{fig:Kripke-lattice}, cf. Remark \ref{rem:KLM-redundant-states}.
\end{remark}

\section{\label{sec:Kripke-Lattice-Models}Kripke Lattice Models}

The models for awareness we construct starts from Kripke models:
\begin{definition}
\label{def:kripke}A \textbf{Kripke model} for $At'\subseteq At$
is a tuple $\mathtt{K}=(W,R,V)$ where $W$ is a non-empty set of
worlds, $R:Ag\rightarrow\mathcal{P}(W^{2})$ assigns to each agent
$a\in Ag$ an accessibility relation denoted $R_{a}$, and $V:At'\rightarrow\mathcal{P}(W)$
is a valuation.

The \textbf{information cell }of $a\in Ag$ at $w\in W$ is $I_{a}(w)=\{v\in W\colon wR_{a}v\}$.
\end{definition}

\noindent The term `information cell' hints at an epistemic interpretation.
For generality, $R$ may assign non-equivalence relations. Some results
explicitly assume otherwise.

As counterpart to the HMS state-space lattice, we build a lattice
of restricted models. The below definition of the set of worlds $W_{X}$
ensures that for any $X,Y\subseteq At$, $X\ne Y$, the sets $W_{X}$
and $W_{Y}$ are disjoint, mimicking the same requirement for state-spaces.
In the restriction $\mathtt{K}_{X}$ of \textsc{$\mathtt{K}$,} it
is required that $(w_{X},v_{X})\in R_{aX}$ iff $(w,v)\in R_{a}$.
Each direction bears similarity to an HMS property: left-to-right
to \ref{HMS5:PPK} and right-to-left to \ref{HMS4:PPI}. They also
remind us, resp., of the \emph{No Miracles} and \emph{Perfect Recall}
properties from Epistemic Temporal Logic, cf. e.g., \cite{Benthem_etal_2009,vanLee_etal_2019}.
\begin{definition}
\label{def:restricted model}Let $\mathtt{K}=(W,R,V)$ be a Kripke
model for $At$. The \textbf{restriction} of $\mathtt{K}$ to $X\subseteq At$
is the Kripke model $\mathtt{K}_{X}=(W_{X},R_{X},V_{X})$ for $X$
where

$W_{X}=\{w_{X}\colon w\in W\}$ where $w_{X}$ is the ordered pair
$(w,X)$,

$R_{Xa}=\{(w_{X},v_{X})\colon(w,v)\in R_{a}\}$ and

$V_{X}:X\rightarrow\mathcal{P}(W_{X})$ such that, for all $p\in X$,$w_{X}\in V_{X}(p)$
iff $w\in V(p)$.

\noindent For the $R_{Xa}$ information cell of $a$ at $w_{X}$,
write $I_{a}(w_{X})$.
\end{definition}

To construct a lattice of restricted models, we simply order them
in accordance with subset inclusion of the atoms. This produces a
complete lattice.
\begin{definition}
\label{def:restriction lattice}Let $\mathtt{K}$ be a Kripke model
for $At$. The \textbf{restriction lattice} of $\mathtt{K}$ is $(\mathcal{K}(\mathtt{K}),\trianglelefteqslant)$
where $\mathcal{K}(\mathtt{K})=\{\mathtt{K}_{X}\}_{X\subseteq At}$
is the set of restrictions of $\mathtt{K}$, and $\mathtt{K}_{X}\trianglelefteqslant\mathtt{K}_{Y}$
iff $X\subseteq Y$.
\end{definition}

Projections in unawareness frames are informally interpreted as mapping
states to alternates of themselves in less expressive spaces. Restriction
lattices offer the same, but implemented w.r.t. $At$: if $Y\subseteq X\subseteq At$,
then $w_{Y}$ is the alternate of $w_{X}$ formally described by the
smaller vocabulary of atoms, $Y$.

The accessibility relations of the Kripke models in a restriction
lattice accounts for the epistemic dimension of the HMS possibility
correspondence $\Pi_{a}$. For the awareness dimension, each agent
$a\in Ag$ is assigned an \emph{awareness map} $\pi{}_{a}$ that maps
a world $w_{X}$ down to $\pi_{a}(w_{X})=w_{Y}$ for some $Y\subseteq X$.
We think of $\pi_{a}(w_{X})$ as $a$'s \emph{awareness image} of
$w_{X}$\textemdash i.e., $w_{X}$ as it occurs to $a$ given her
(un)awareness; the submodel from $\pi_{a}(w_{X})$ is thus $a$'s
subjective perspective.

In the following definition, we introduce three properties of awareness
maps, which we will assume. Intuitions follow the definition.
\begin{definition}
\label{def:awareness.map}With $\mathsf{L}=(\mathcal{K}(\mathtt{K}),\trianglelefteqslant)$
a restriction lattice, let $\Omega_{\mathsf{L}}=\bigcup\mathcal{K}(\mathtt{K})$
and let $\pi$ assign to each agent $a\in Ag$ an \textbf{awareness
map} $\pi_{a}:\Omega_{\mathsf{L}}\rightarrow\Omega_{\mathsf{L}}$
satisfying
\begin{description}
\item [{\label{our1-DownwardsProjection}D}] (\textbf{Downwards}) For all
$w_{X}\in\Omega_{\mathsf{L}}$, $\pi{}_{a}(w_{X})=w_{Y}$ for some
$Y\subseteq X$.
\item [{\label{our2-Intro.Idem}I\,I}] (\textbf{Introspective Idempotence})
If $\pi_{a}(w_{X})=w_{Y}$, then for all $v_{Y}\in I_{a}(w_{Y})$,
$\pi_{a}(v_{Y})=u_{Y}$ for some $u_{Y}\in I_{a}(w_{Y})$.
\item [{\label{our3-NoSurpises}NS}] (\textbf{No Surprises}) If $\pi_{a}(w_{X})=w_{Z}$,
then for all $Y\subseteq X$, $\pi_{a}(w_{Y})=w_{Y\cap Z}$.\vspace{-3pt}
\end{description}
Call $\mathsf{K}=(\mathcal{K}(\mathtt{K}),\trianglelefteqslant,\pi)$
the \textbf{Kripke lattice model} of $\mathtt{K}$.
\end{definition}

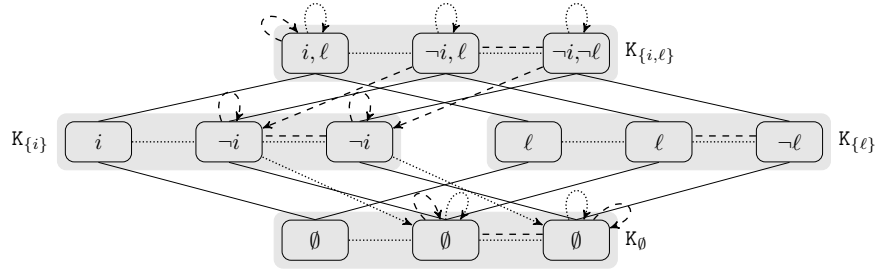
\begin{figure}
\vspace{-28pt}
\resizebox{1\textwidth}{!}{

\pgfdeclarelayer{background}
\pgfdeclarelayer{foreground}
\pgfsetlayers{background,main,foreground}

\begin{tikzpicture}
\tikzset{world/.style={rectangle, draw=black, rounded corners, text width=23pt, minimum height=18pt, 
		text centered},
	modal/.style={>=stealth',shorten >=1pt,shorten <=1pt,auto,node distance=1cm},
	state/.style={circle,draw,inner sep=0.5mm,fill=black}, 
	PCon/.style ={rectangle,draw=black, rounded corners, semithick, densely dotted, text width=30pt, minimum height=18pt, text centered},
	PCbn/.style ={draw=black, rounded corners, semithick, dashed, text width=34pt, minimum height=21pt}, 
	PCo/.style={->,densely dotted, semithick, >=stealth'},
	EPCo/.style={-,densely dotted, semithick},
	EPCb/.style={-,dashed, semithick},
	PCb/.style={->,dashed, semithick, >=stealth'},
	proj/.style={-, line width=0.001pt, draw=black!2},
	block/.style ={rectangle, fill=black!10,draw=black!10, rounded corners, text width=5em, minimum height=0.3cm},
	reflexive above/.style={->,out=60,in=100,looseness=8},
	reflexive left/.style={->,out=-170,in=-190,looseness=8},
}

\begin{pgfonlayer}{foreground}
\node[world] (a) {$i,\ell$};
\node[world, right= of a] (b) {$\neg i, \ell$};
\node[world, right= of b] (c) {$\neg i,\!\!\neg \ell$};

\node[world, below =of a, xshift=7mm, yshift=2.5mm] (c') {$\neg i$};
\node[world, left=of c'] (b') {$\neg i$};
\node[world, left=of b'] (a') {$ i$};

\node[world, right=of c', xshift=6mm] (a'') {$\ell$};
\node[world, right=of a''] (b'') {$ \ell$};
\node[world, right=of b''] (c'') {$\neg \ell$};

\node[world, below=of b, yshift=-13mm] (b''') {$\emptyset$};
\node[world, right=of b'''] (c''') {$\emptyset$};
\node[world, left=of b'''] (a''') {$\emptyset$};
\end{pgfonlayer}

\node[block, fit=(a)(b)(c)] (block) [label= right:{$\mathtt{K}_{\{i,\ell\}}$}]{ };
\node[block, fit=(a')(b')(c')] (block) [label=left: {$\mathtt{K}_{\{i\}}$}]{ };

\node[block, fit=(a'')(b'')(c'')] (block) [label= right:{$\mathtt{K}_{\{\ell\}}$}]{ };

\node[block, fit=(a''')(b''')(c''')] (block) [label= right:{$\mathtt{K}_{\emptyset}$}]{ };

\draw[PCo] (a) to [out=110,in=70,looseness=7] (a);
\draw[PCo] (b) to [out=110,in=70,looseness=7] (b);
\draw[PCo] (c) to [out=110,in=70,looseness=7] (c);	
\draw[EPCb] (c) to [out=170,in=10,looseness=0] (b);	
\draw[EPCb] (c') to [out=170,in=10,looseness=0] (b');	
\draw[EPCb] (c'') to [out=170,in=10,looseness=0] (b'');	
\draw[EPCb] (c''') to [out=170,in=10,looseness=0] (b''');	
\draw[PCb] (a) to [out=-200,in=130,looseness=6] (a)  node {$ $};
\draw[PCb] (b) to  (b');	
\draw[PCb] (b') to [out=110,in=70,looseness=7] (b');
\draw[PCb] (c) to (c');	
\draw[PCb] (c') to [out=110,in=70,looseness=7] (c');	
\draw[PCo] (b') to  (b''');	
\draw[PCo] (c') to (c''');	
\draw[PCb] (b''') to [out=135,in=105,looseness=7] (b''');
\draw[PCo] (b''') to [out=90,in=55,looseness=7] (b''');
\draw[PCo] (c''') to [out=110,in=70,looseness=7] (c''');
\draw[PCb] (c''') to [out=50,in=20,looseness=6] (c''');

\draw[EPCo] (a) to (b);
\draw[EPCo] (b) to (c);
\draw[EPCo] (a') to (b');
\draw[EPCo] (b') to (c');
\draw[EPCo] (a'') to (b'');
\draw[EPCo] (b'') to (c'');
\draw[EPCo] (a''') to (b''');
\draw[EPCo] (b''') to (c''');

\draw [-] (a.south) -- (a''.north)[line width=0.001pt];
\draw [-] (b.south) -- (b''.north)[line width=0.001pt];
\draw [-] (c.south) -- (c''.north)[line width=0.001pt];

\draw [-] (a.south) -- (a'.north)[line width=0.001pt];
\draw [-] (b.south) -- (b'.north)[line width=0.001pt];
\draw [-] (c.south) -- (c'.north)[line width=0.001pt];

\draw [-] (a'.south) -- (a'''.north)[line width=0.001pt];
\draw [-] (b'.south) -- (b'''.north)[line width=0.001pt];
\draw [-] (c'.south) -- (c'''.north)[line width=0.001pt];

\draw [-] (a''.south) -- (a'''.north)[line width=0.001pt];
\draw [-] (b''.south) -- (b'''.north)[line width=0.001pt];
\draw [-] (c''.south) -- (c'''.north)[line width=0.001pt];
\end{tikzpicture}
}

\caption{\label{fig:Kripke-lattice}A Kripke lattice model of the Figure \ref{fig:HMSmodel}
example. See four restrictions (gray rectangles), ordered spatially
as a lattice. States (smallest rectangles) are labeled with their
true literals, over the set $At=\{i,\ell\}$. Horizontal dashed and
dotted lines \emph{inside restrictions} represent Buyer and Owner's
accessibility relations (omitted are links obtainable by reflexive-transitive
closure), while dotted and dashed arrows \emph{between restrictions}
represent their awareness maps (some arrows are omitted: they go to
states' alternates in $\mathtt{K}_{\emptyset}$, and are irrelevant
from $(i,l)$). Thin lines connect states with their alternate in
lower restrictions. See also Remark \ref{rem:KLM-redundant-states}
concerning~$\mathtt{K}_{\{\ell\}}$. }
\end{figure}

\ref{our1-DownwardsProjection} ensures that an agent's awareness
image of a world is a restricted representation of that same world.
Hence the awareness image does not conflate worlds, and does not allow
the agent to be aware of a more expressive vocabulary than that which
describes the world she views from. With \ref{our2-Intro.Idem} and
accessibility assumed reflexive, it entails that $\pi_{a}$ is idempotent:
for all $w_{X},$ $\pi_{a}(\pi_{a}(w_{X}))=\pi_{a}(w_{X})$. Alone,
\ref{our2-Intro.Idem} states that in her awareness image, the agent
knows, and is aware of, the atoms that she is aware of. Given that
accessibility is distributed by inheritance through the Kripke models
in restriction lattices, the property implies that the same holds
for every such model. \ref{our3-NoSurpises} guarantees that awareness
remains ``consistent'' down the lattice, so that awareness of an
atom does not appear or disappear without reason. Consider the consequent
$\pi_{a}(w_{Y})=w_{Y\cap Z}$ and its two subcases $\pi_{a}(w_{Y})=w_{Y^{*}}$
with $Y^{*}\subseteq Y\cap Z$ and $Y^{*}\supseteq Y\cap Z$. Colloquially,
the first states that if atoms are removed from the description of
the world from which the agent views, then they are also removed from
her awareness. Oppositely, the second states that if atoms are removed
from the description of the world from which the agent views, then
no more than these should be removed from her awareness. Jointly,
no awareness should ``miraculously'' appear, and all \mbox{awareness
should be ``recalled''.}\footnote{Again, we are reminded of \emph{No Miracles }and \emph{Perfect Recall}.}
\begin{remark}
Contrary to HMS models (cf. Remark \ref{rem:No-Objective-State}),
Kripke lattice models have an objective perspective: designating an
`actual world' in $\mathtt{K}_{At}$ allows one to check agents'
uncertainty about the possible states of the world described by the
maximal language, i.e., from $\mathtt{K}_{At}$ we can read off their
``actual implicit knowledge''. See e.g. Figure~\ref{fig:Kripke-lattice}:
In the $(\neg i,\ell)$ state, the dashed agent's ``true'' uncertainty
about $\ell$ \emph{is }determined, contrary to the same state in
the HMS model of Figure~\ref{fig:HMSmodel}.
\end{remark}

\begin{remark}
\label{rem:KLM-redundant-states} In Remark \ref{rem:HMS-redundant-states},
we mentioned that the HMS state-space $S_{\{\ell\}}$ of Figure~\ref{fig:HMSmodel}
is redundant. Similarly, $\mathtt{K}_{\{\ell\}}$ is redundant in
Figure~\ref{fig:Kripke-lattice} (from $(i,\ell)$, $\mathtt{K}_{\{\ell\}}$
is unreachable.) However, contrary to the HMS case, it is here required
by definition, as a restriction lattice contains all restrictions
of the original Kripke model. For simplicity of constructions, we
have not here attempted to prune away redundant restrictions. A more
general model class may be obtained by letting models be based on
sub-orders of the restriction lattice. See also the concluding~remarks.
\end{remark}

\section{\label{sec:Moving}Moving between HMS Models and Kripke Lattices}

To clarify the relationship between HMS models and Kripke lattice
models, we introduce transformations between the two model classes,
showing that a model from one class encodes the structure of a model
from the other. The core idea is to think of a possibility correspondence
$\Pi_{a}$ as the composition of $I_{a}$ and $\pi_{a}$: $\Pi_{a}(w)$
is the information cell of the awareness image of $w$.

The propositions of this section show that the transformations produce
models of the desired class. Additionally, their proofs shed partial
light on the relationship between the HMS properties and those assumed
for awareness maps $\pi_{a}$ and accessibility relations $R_{a}$:
we discuss this shortly in the concluding remarks.

\subsection{From HMS Models to Kripke Lattice Models}

Moving from HMS models to Kripke lattice models requires a somewhat
involved construction as it must tease apart unawareness and uncertainty
from the possibility correspondences, and track the distribution of
atoms and their relationship to awareness. For an example, then the
Kripke lattice model in Figure \ref{fig:Kripke-lattice} is the HMS
model of Figure~\ref{fig:HMSmodel} transformed.
\begin{definition}
\label{def:L-transform}Let $\mathsf{M}=(\mathcal{S},\preceq,\mathcal{R},\Pi,V_{\mathsf{M}})$
be an HMS model with maximal state-space $T$. For any $O\subseteq\Omega_{\mathsf{M}}$,
let $At(O)=\{p\in At\colon O\subseteq V_{M}(p)\cup\neg V_{M}(p)\}$.\footnote{$At(O)$ contains the atoms that have a defined truth value in every
$s\in O$.}\smallskip{}

\noindent The \textbf{$L$-transform model of $\mathsf{M}$} is $L(\mathsf{M})=(\mathcal{K}(\mathtt{K}),\trianglelefteqslant,\pi)$
where the Kripke model $\mathtt{K}=(W,R,V)$ for $At$ given by\smallskip{}

\noindent \quad{}$W=T$;

\noindent \quad{}$R$ maps each $a\in Ag$ to $R_{a}\subseteq W^{2}$
s.t. $(w,v)\in R_{a}$ iff $r_{S(\Pi_{a}(w))}^{T}(v)\in\Pi_{a}(w)$;

\noindent \quad{}$V:At\rightarrow\mathcal{P}(W)$, defined by $V(p)\ni w$
iff $w\in V_{\mathsf{M}}(p)$, for every $p\in At$;\smallskip{}

\noindent $\pi$ assigns each $a\in Ag$ a map $\pi_{a}:\Omega_{L(\mathsf{M})}\rightarrow\Omega_{L(\mathsf{M})}$
s.t. for all $w_{X}\in\Omega_{L(\mathsf{M})}$,\smallskip{}

\noindent \quad{}$\pi_{a}(w_{X})=w_{Y}$ where $Y=At(S_{Y})$ for
the $S_{Y}\in\mathcal{S}$ with $S_{Y}\supseteq\Pi_{a}(r_{S_{X}}^{T}(w))$

\noindent \quad{}where $S_{X}=\min\{S\in\mathcal{S}\colon At(S)=X\}$.\smallskip{}

\noindent The \textbf{state correspondence} between $\mathsf{M}$
and $L(\mathsf{M})$ is the map $\ell:\Omega_{\mathsf{M}}\rightarrow2^{\Omega_{L(\mathsf{M})}}$
s.t. for all $s\in\Omega_{\mathsf{M}}$\smallskip{}

\noindent \quad{}$\ell(s)=\{w_{X}\in W_{X}\colon w\in(r_{S(s)}^{T})^{-1}(s)\text{ for }X=At(S(s))\}.$
\end{definition}

Intuitively, in the $L$-transform model, a world $v\in W$ is accessible
from a world $w\in W$ for an agent if, and only if, $v$'s restriction
to the agent's vocabulary at $w$ is one of the possibilities she
entertains.\footnote{We thank a reviewer for this wording.} In addition,
the awareness map $\pi_{a}$ of agent $a$ relates a world $w_{X}$
to its less expressive counterpart $w_{Y}$ if, and only if, $Y$
is the vocabulary agent $a$ adopts when describing what she considers
possible.
\begin{remark}
The $L$-transform model $L(\mathsf{M})$ of $\mathsf{M}$ is well-defined
as the object $\mathtt{K}=(W,R,V)$ is in fact a Kripke model for
$At$: $i)$ By def. of HMS models, $W=T\in\mathcal{S}$ is non-empty;
$ii)$ for each $a$, $R_{a}\subseteq W^{2}$ is well-defined: if
$w\in T=W$, then by \ref{HMS1:confinement}, $\Pi_{a}(w)\subseteq S$,
for some $S\in\mathcal{S}$. Hence, $U=\{v\in T\colon r_{S}^{T}(v)\in\Pi_{a}(w)\}$
is well-defined, and so is $\{(w,v)\in T^{2}\colon v\in U\}=R_{a}$;
$iii)$ As $V_{\mathsf{M}}$ is an HMS valuation $V_{\mathsf{M}}:At\rightarrow\Sigma$
for $At$, clearly $V$ is valuation for $At$. Hence $\mathtt{K}=(W,R,V)$
is a Kripke model for $At$.
\end{remark}

\begin{remark}
The $\min$ used in defining $S_{X}$ is due to the issue of Remark
\ref{rem:HMS-valuations-only}.
\end{remark}

\begin{remark}
\label{rem:state-correspondence}The state correspondence map $\ell$
is also well-defined. That it maps each state in $\Omega_{\mathsf{M}}$
to a \emph{set }of worlds in $\Omega_{L\mathsf{(M)}}$ points to a
construction difference between HMS models and Kripke lattice models:
in the former, the downwards projections of two states may `merge'
them, so state-spaces may shrink when moving down the lattice; in
the latter, distinct worlds remain distinct, so all world sets in
a restriction lattice share cardinality.
\end{remark}

As unawareness and uncertainty are separated in Kripke lattice models,
we show two results about $L$-transforms. The first shows that the
\ref{HMS1:confinement}, \ref{HMS3:stationarity} and \ref{HMS5:PPK}
entail that $\pi_{a}$ assigns awareness maps, and the second that
the five HMS properties entail that $R$ assigns equivalence relations.
In showing the first, we make use of the following lemma, which intuitively
shows that the information cell of an agent contains a state described
with a certain vocabulary if, and only if, the agent considers possible
the corresponding state described with the same vocabulary:
\begin{lemma}
\label{LemmaYYY} For every $w_{Y}\in\Omega_{\mathsf{K}}$, if $\Pi_{a}(w)\subseteq S$
and $At(S)=Y$, then $v_{Y}\in I_{a}(w_{Y})$ iff $v_{S}\in\Pi_{a}(w)$.
\end{lemma}

\begin{proof}
Let $w_{Y}\in\Omega_{L\mathsf{(M)}}$. Consider the respective $w\in T=W$
and let $\Pi_{a}(w)\subseteq S$, with $At(S)=Y$. Assume that $v_{Y}\in I_{a}(w_{Y})$.
This is the case iff (def. of $I_{a}$) $(w_{Y},v_{Y})\in R_{Ya}$
iff (def. of restriction lattice) $(w,v)\in R_{a}$ iff (Def. \ref{def:L-transform})
$v_{S}\in\Pi_{a}(w)$.
\end{proof}

\begin{proposition}
\label{prop:L-transform-is-KLM}For any HMS model $\mathsf{M}$, its
$L$-transform $L(\mathsf{M})$ is a Kripke lattice model.
\end{proposition}

\begin{proof}
Let $\mathsf{M}=(\mathcal{S},\preceq,\mathcal{R},\Pi,V_{\mathsf{M}})$
be an HMS model with maximal state-space $T$. We show that $L(\mathsf{M})=(\mathcal{K}(\mathtt{K}),\trianglelefteqslant,\pi)$
is a Kripke lattice model by showing that $\pi_{a}$ satisfies the
three properties of an awareness map:

\emph{\ref{our1-DownwardsProjection}}: Consider an arbitrary $w_{X}\in\Omega_{L(\mathsf{M})}$.
By def. of $L$-transform, $X=At(S)$ for some $S\in\mathcal{S}$.
Let $S_{X}=\min\{S\in\mathcal{S}\colon At(S)=X\}$. If $w_{X}\in W_{X}$
then for some $w\in W=T$, $w_{S_{X}}\in S_{X}$. By \ref{HMS1:confinement},
$\Pi_{a}(w_{S_{X}})\subseteq S_{Y}$, for some $S_{Y}\preceq S_{X}$.
Let $Y=At(S_{Y})$. Then, by def. of $\pi_{a}$, $\pi_{a}(w_{X})=w_{Y}$
and $Y\subseteq X$.

\emph{\ref{our2-Intro.Idem}}: Let $\pi_{a}(w_{X})=w_{Y}$. By def.
of $\pi_{a}$, it holds that $\Pi_{a}(r_{S_{X}}^{T}(w))\subseteq S_{Y}$
with $At(S_{Y})=Y$ and $S_{X}=\min\{S\in\mathcal{S}\colon At(S)=X\}$.
For a contradiction, suppose there exists a $v_{Y}\in I_{a}(w_{Y})$
s.t. for all $u_{Y}\in I_{a}(w_{Y})$, $\pi_{a}(v_{Y})\not=u_{Y}$.
Then $\pi_{a}(v_{Y})=t_{Z}$ for some $Z\subseteq Y$ and $t_{Z}\not\in I_{a}(w_{Y})$.
By def. of $\pi_{a}$, $\pi_{a}(v_{Y})=t_{Z}$ iff $\Pi_{a}(r_{S_{Y}}^{T}(v))\subseteq S_{Z}$,
where $Z=At(S_{Z})$. Then, by Lemma \ref{LemmaYYY}, $t_{Z}\in I_{a}(v_{Z})$
iff $t_{S_{Z}}\in\Pi_{a}(r_{S_{Y}}^{T}(v))$. Moreover, as $\Pi_{a}(r_{S_{X}}^{T}(w))\subseteq S_{Y}$
and $At(S_{X})=X$, by Lemma \ref{LemmaYYY}, it also follows that
$v_{Y}\in I_{a}(w_{Y})$ iff $v_{S_{Y}}\in\Pi_{a}(r_{S_{X}}^{T}(w))$.
Since $v_{Y}\in I_{a}(w_{Y})$ then $v_{S_{Y}}\in\Pi_{a}(r_{S_{X}}^{T}(w))$.
Hence, by \ref{HMS3:stationarity}, $\Pi_{a}(r_{S_{X}}^{T}(w))=\Pi_{a}(r_{S_{Y}}^{T}(v))$,
which implies $t_{S_{Z}}\in\Pi_{a}(r_{S_{X}}^{T}(w))$. But then $t_{Z}\in I_{a}(v_{Z})$,
contradicting the assumption that $t_{Z}\not\in I_{a}(w_{Y})$. Thus,
for all $v_{Y}\in I_{a}(w_{Y})$, $\pi_{a}(v_{Y})=u_{Y}$ for some
$u_{Y}\in I_{a}(w_{Y})$.

\emph{\ref{our3-NoSurpises}}: Let $\pi_{a}(w_{X})=w_{Y}$. By \ref{our1-DownwardsProjection}
(cf. item 1. above), $Y\subseteq X$. Consider an arbitrary $Z\subseteq X$.
We have two cases: either $i)$ $Z\subseteq Y$ or $ii)$ $Y\subseteq Z$.
$i)$: then $Z\subseteq Y\subseteq X$. Let $Z=At(S_{Z})$, $Y=At(S_{Y})$,
and $X=At(S_{X})$. Then $S_{Z}\preceq S_{Y}\preceq S_{X}$. By \ref{HMS5:PPK},
$\big(\Pi_{a}(r_{S_{X}}^{T}(w))\big)_{Z}=\Pi_{a}(r_{S_{Z}}^{T}(w))$.
As $\pi_{a}(w_{X})=w_{Y}$, by def. of $\pi_{a}$, $\Pi_{a}(r_{S_{X}}^{T}(w))\subseteq S_{Y}$.
Then $\big(\Pi_{a}(r_{S_{X}}^{T}(w))\big)_{Z}=r_{S_{Z}}^{S_{Y}}\big(\Pi_{a}(r_{S_{X}}^{T}(w))\big)\subseteq S_{Z}$.
Hence $\Pi_{a}(r_{S_{Z}}^{T}(w))\subseteq S_{Z}$, and by def. of
$\pi_{a}$, $\pi_{a}(w_{Z})=w_{Z}$. As $Z\subseteq Y$, $\pi_{a}(w_{Z})=w_{Z}=w_{Z\cap Y}$.
$ii)$: then $Y\subseteq Z\subseteq X$. By analogous reasoning, we
have $\pi_{a}(w_{Y})=w_{Y}=w_{Y\cap Z}$ as $Y\subseteq Z$. We can
conclude that if $\pi_{a}(w_{X})=w_{Y}$, then for all $Z\subseteq X$,
$\pi_{a}(w_{Z})=w_{Z\cap X}$.
\end{proof}

\begin{proposition}
\label{prop:L-transform-has-EQ-R}If $L(\mathsf{M})=(\mathcal{K}(\mathtt{K}=(W,R,V)),\trianglelefteqslant,\pi)$
is the $L$-transform of an HMS model $\mathsf{M}$, then for every
$a\in Ag$, $R_{a}$ is an equivalence relation.
\end{proposition}

\begin{proof}
\noindent Let $\mathsf{M}=(\mathcal{S},\preceq,\mathcal{R},\Pi,V_{\mathsf{M}})$
have maximal state-space~$T$.

\emph{Reflexivity:} Let $w\in T$ and $\Pi_{a}(w)\subseteq S$, for
some $S\in\mathcal{S}$. By def. of upwards closure, $(\Pi_{a}(w))^{\uparrow}=\bigcup_{S'\succeq S}(r_{S}^{S'})^{-1}(\Pi_{a}(w))$,
and by \ref{HMS2:G.Ref}, $w\in(\Pi_{a}(w))^{\uparrow}=\bigcup_{S'\succeq S}(r_{S}^{S'})^{-1}(\Pi_{a}(w))$.
Since $T\succeq S$, then $r_{S}^{T}(w)\in\Pi_{a}(w)$. Thus, $(w,w)\in R_{a}$,
by def. $L$-transform. By def. of restriction lattices, this holds
for all $A\subseteq At$, i.e. $(w_{A},w_{A})\in R_{Aa}$.

\emph{Transitivity:} Let $w,v,u$ be in $T$. By \ref{HMS1:confinement},
there are $S,S'\in\mathcal{S}$ such that $\Pi_{a}(w)\subseteq S$
and $\Pi_{a}(v)\subseteq S'$. Assume that $(w,v)\in R_{a}$ and $(v,u)\in R_{a}$.
By def. of $R_{a}$, then $r_{S}^{T}(v)\in\Pi_{a}(w)$ and $r_{S'}^{T}(u)\in\Pi_{a}(v)$.
By \ref{HMS3:stationarity}, $\Pi_{a}(w)=\Pi_{a}(r_{S}^{T}(v))$ and
$\Pi_{a}(v)=\Pi_{a}(r_{S'}^{T}(u))$. As $v\in T$ and $S\preceq T$,
by \ref{HMS4:PPI}, $\Pi_{a}(v)^{\uparrow}\subseteq\Pi_{a}(r_{S}^{T}(v))^{\uparrow}=\Pi_{a}(w)^{\uparrow}$.
Hence, as $r_{S'}^{T}(u)\in\Pi_{a}(v)^{\uparrow}$, also $r_{S'}^{T}(u)\in\Pi_{a}(w)^{\uparrow}$.
By def. of upwards closure, $r_{S}^{T}(u)\in\Pi_{a}(w)$. Finally,
$(w,u)\in R_{a}$ by def. of $R_{a}$.

\emph{Symmetry:} Let $w,v\in T$ be in $T$. Assume that $(w,v)\in R_{a}$.
By \ref{HMS1:confinement}, there are $S,S'\in\mathcal{S}$ such that
$\Pi_{a}(w)\subseteq S$ and $\Pi_{a}(v)\subseteq S'$. Then $r_{S}^{T}(v)\in\Pi_{a}(w)$
(def. of $L$-transform), and by \ref{HMS3:stationarity}, $\Pi_{a}(w)=\Pi_{a}(r_{S}^{T}(v))$.
As $v\in T$ and $T\succeq S$, by \ref{HMS4:PPI}, by $\Pi_{a}(v)^{\uparrow}\subseteq\Pi_{a}(r_{S}^{T}(v))^{\uparrow}$.
Then, by def. of upwards closure, $T\text{\ensuremath{\succeq S'\succeq S}}$.
As $v\in T$, by \ref{HMS5:PPK}, $r_{S}^{S'}(\Pi_{a}(v))=\Pi_{a}(r_{S}^{T}(v))$.
By \ref{HMS2:G.Ref}, $x\in\Pi_{a}(w)^{\uparrow}$, and since $\Pi_{a}(w)\subseteq S$
then $r_{S}^{T}(w)\in\Pi_{a}(w)$, by def. of upward closure. Then
$r_{S}^{T}(w)\in\Pi_{a}(w)=\Pi_{a}(r_{S}^{T}(v))=r_{S}^{S'}(\Pi_{a}(v))$.
So $r_{S}^{T}(w)\in r_{S}^{S'}(\Pi_{a}(v))$, i.e. $r_{S'}^{T}(w)\in\Pi_{a}(v)$,
by def. of $r$. Hence, $(v,w)\in R_{a}$, by def. of $R_{a}$.
\end{proof}

\subsection{From Kripke Lattice Models to HMS Models}

Moving from Kripke lattice models to HMS models requires a less involved
construction, as the restriction lattice almost encode projections,
and unawareness and uncertainty are simply composed to form possibility
correspondences:
\begin{definition}
\label{def:H-transform}Let $\mathsf{K}=(\mathcal{K}(\mathtt{K}=(W,R,V)),\trianglelefteqslant,\pi)$
be a Kripke lattice model for $At$. The \textbf{$H$-transform} of
\textbf{$\mathsf{K}$} is $H(\mathsf{K})=(\mathcal{S},\preceq,\mathcal{R},\Pi,V_{H(\mathsf{K})})$
where

\noindent $\mathcal{S}=\{W_{X}\subseteq\Omega_{\mathsf{K}}:\mathtt{K}_{X}\in\mathcal{K}(\mathtt{K})\}$;

\noindent $W_{X}\preceq W_{Y}$ iff $\mathtt{K}_{X}\trianglelefteqslant\mathtt{K}_{Y}$;

\noindent $\mathcal{R}=\{r_{W_{Y}}^{W_{X}}\colon r_{W_{Y}}^{W_{X}}(w_{X})=w_{Y}\text{ for all }w\in W,\text{ and all }X,Y\subseteq At\}$;

\noindent $\Pi=\{\Pi_{a}\in(2^{\Omega_{\mathsf{K}}})^{\Omega_{\mathsf{K}}}\colon\Pi_{a}(w_{X})=I_{a}(\pi_{a}(w_{X}))\text{ for all }w\in W,X\subseteq At,a\in Ag\}$;

\noindent $V_{H(\mathsf{K})}(p)=\{w_{X}\in\Omega_{\mathsf{K}}\colon X\ni p\text{ and }w_{X}\in V_{X}(p)\}$
for all $p\in At$.
\end{definition}

As HMS models lump together unawareness and uncertainty, we show only
one result in this direction:
\begin{proposition}
\label{prop:H-transform-is-HMS}For any Kripke lattice model $\mathsf{K}=(\mathcal{K}(\mathtt{K}=(W,R,V)),\trianglelefteqslant,\pi)$
s.t. $R$ assigns equivalence relations, the $H$-transform $H(\mathsf{K})$
is an HMS~model.
\end{proposition}

\begin{proof}
Let $\mathsf{K}$ be as stated and let $H(\mathsf{K})=(\mathcal{S},\preceq,\mathcal{R},\Pi,V_{H(\mathsf{K})})$
be its $H$-transform.

\noindent $\mathcal{S}=\{W_{X},W_{Y},...\}$ is composed of non-empty
disjoint sets by construction and $(\mathcal{S},\preceq)$ is a complete
lattice as $(\mathcal{K}(\mathtt{K}),\trianglelefteqslant)$ is so.
$\mathcal{R}$ is clearly a family of well-defined, surjective and
commutative projections. As $\Pi$ assigns to each $a\in Ag$, $\Pi_{a}(w_{X})=I_{a}(\pi_{a}(w_{X}))$,
for all $w\in W$, $X\subseteq At$, it assigns $a$ a map $\Pi_{a}:\Omega_{H(\mathsf{K})}\rightarrow2^{\Omega_{H(\mathsf{K})}}$,
which is a possibility correspondence as it satisfies the HMS properties:

\emph{\ref{HMS1:confinement}}: For $w_{X}\in W_{X}$, $\Pi_{a}(w_{X})=I_{a}(\pi_{a}(w_{X}))$,
by Def. \ref{def:H-transform}. By \ref{our1-DownwardsProjection},
$\pi_{a}(w_{X})=w_{Y}$ for some $Y\subseteq X$, and $I_{a}(\pi_{a}(w_{X}))=I_{a}(w_{Y})$.
So, $\Pi_{a}(w_{X})\subseteq W_{Y}$ for some $Y\subseteq X$.

\emph{\ref{HMS2:G.Ref}}:\textbf{ }Let $w_{X}\in\Omega_{\mathsf{K}}$,
$X\subseteq At.$ By \ref{our1-DownwardsProjection}, $\pi_{a}(w_{X})=w_{Y}$
for some $Y\subseteq X$. By def. of $\Pi_{a}$ and $I_{a}$, $\Pi_{a}(w_{X})=I_{a}(w_{Y})=\{v_{Y}\in\Omega_{\mathsf{K}}:(w_{Y},v_{Y})\in R_{Ya}\}$.
Hence $\Pi_{a}(w_{X})\subseteq W_{Y}$. By def. of upward closure,
$(\Pi_{a}(w_{X}))^{\uparrow}=(I_{a}(w_{Y})){}^{\uparrow}=\{u_{Z}\in\Omega_{\mathsf{K}}:Y\subseteq Z\text{ and }u_{Y}\in\{v_{Y}\in\Omega_{\mathsf{K}}:(w_{Y},v_{Y})\in R_{Ya}\}\}$,
with the last identity given by the def. of $r_{W_{Y}}^{W_{Z}}$.
As $R_{a}$ is an equivalence relation, so is $R_{Ya}$, by def. So
$w_{Y}\in\{v_{Y}\in\Omega_{\mathsf{K}}:(w_{Y},v_{Y})\in R_{Ya}\}$,
and since $Y\subseteq X$, then \textbf{$w_{X}\in(\Pi_{a}(w_{X})){}^{\uparrow}$}.

\emph{\ref{HMS3:stationarity}}:\textbf{ }For $w_{X}\in\Omega_{\mathsf{K}}$,
assume $v\in\Pi_{a}(w_{X})=I_{a}(\pi_{a}(w_{X}))$. By \ref{our1-DownwardsProjection},
$v\in I_{a}(w_{Y})$, for some $Y\subseteq X$. With $R_{Ya}$ an
equivalence relation, $v\in I_{a}(w_{Y})$ iff $w_{Y}\in I_{a}(v)$,
i.e., $I_{a}(v)=I_{a}(w_{B})$. \ref{our2-Intro.Idem} and \ref{our1-DownwardsProjection}
entails that for all $u_{Y}\in I_{a}(w{}_{Y})$, $\pi_{a}(u_{Y})=u_{Y}$,
so $\pi_{a}(v)=v$. Therefore $\Pi_{a}(v)=I_{a}(\pi_{a}(v))=I_{a}(v)=I_{a}(w_{Y})=I_{a}(\pi_{a}(w_{X}))=\Pi_{a}(w_{X})$.
Thus, if $v\in\Pi_{a}(w_{X})$, then $\Pi_{a}(v)=\Pi_{a}(w_{X})$.

\emph{\ref{HMS4:PPI}}:\textbf{ }Let $w_{X}\in W_{X}$ and $W_{Y}\preceq W_{X}$,
i.e. $Y\subseteq X\subseteq At$. Let $q_{Q}\in(\Pi_{a}(w_{X}))^{\uparrow}$
with $Q\subseteq At$. By def. of $\Pi_{a}$ and \ref{our1-DownwardsProjection},
$\Pi_{a}(w_{X})=I_{a}(\pi_{a}(w_{X}))=I_{a}(w_{Z})$ for some $Z\subseteq X$.
By def. of upwards closure, it follows that $q_{Z}\in I_{a}(w_{Z})=\Pi_{a}(w_{X})$.
Now let $\pi_{a}(w_{Y})=w_{P}$ for some $P\subseteq Y$. Then, by
\ref{our3-NoSurpises}, $P=Z\cap Y$, so $P\subseteq Z$. As $q_{Z}\in I_{a}(w_{Z})$,
then $q_{P}\in I_{a}(w_{P})=I_{a}(\pi_{a}(w_{Y}))=\Pi_{a}(w_{Y})$,
by def. of restriction lattice. Since $q_{Q}\in(\Pi_{a}(w_{X}))^{\uparrow}=(I_{a}(w_{Z}))^{\uparrow}$,
then $Z\subseteq Q$. It follows that $P\subseteq Z\subseteq Q$,
which implies $q_{Q}\in(\Pi_{a}(w_{Y}))^{\uparrow}$. Hence, if $q_{Q}\in(\Pi_{a}(w_{X}))^{\uparrow}$,
then $q_{Q}\in(\Pi_{a}(w_{Y}))^{\uparrow}$, i.e., $(\Pi_{a}(w_{X}))^{\uparrow}\subseteq(\Pi_{a}(w_{Y}))^{\uparrow}$.

\emph{\ref{HMS5:PPK}}:\textbf{ }Suppose that $W_{Z}\preceq W_{Y}\preceq W_{X}$,
$w_{X}\in W_{X}$ and $\Pi_{a}(w_{X})\subseteq W_{Y}$, i.e. $\Pi_{a}(w_{X})=I_{a}(w_{Y})$
and $\pi_{a}(w_{X})=w_{Y}$. As $Z\subseteq Y\subseteq X$, \ref{our3-NoSurpises}
implies $\pi_{a}(w_{Z})=w_{Z\cap Y}=w_{Z}$. Hence, $\Pi_{a}(w_{Z})=I_{a}(w_{Z})\subseteq W_{Z}$.
Hence \ref{HMS5:PPK} is established if $\left(I_{a}(w_{Y})\right)_{Z}=I_{a}(w_{Z})$.
As $\left(I_{a}(w_{Y})\right)_{Z}=\{x_{Z}\in\Omega_{\mathsf{K}}:x_{Y}\in I_{a}(w_{Y})\}$,
then clearly $\left(I_{a}(w_{Y})\right)_{Z}=I_{a}(w_{Z})$. Thus,
$\left(\Pi_{a}(w_{X})\right)_{Z}=\Pi_{a}(w_{Z})$.

Finally, $V_{H(\mathsf{K})}$ is an HMS valuation as for each $p\in At$,
$V_{H(\mathsf{K})}(p)$ is an event $(D^{\uparrow},S)$ with $D=\{w_{\{p\}}\in W_{\{p\}}:w_{\{p\}}\in V_{\{p\}}(p)\}$
and $S=W_{\{p\}}$.\vspace{-3pt}
\end{proof}

\section{\label{sec:Language}Language for Awareness and Model Equivalence}

Multiple languages for knowledge and awareness exist. The Logic of
General Awareness (LGA, \cite{HalpernFagin88}) takes implicit knowledge
and awareness as primitives, and define explicit knowledge as `implicit
knowledge $\wedge$ awareness'; other combinations are discussed
in \cite{Velazquez-Quesada2010-BENTDO-8}. Variations of LGA include
quantification over objects \cite{BoardChung2006}, formulas \cite{halpern2009,halpern2013reasoning,Agotnes2014},
and even unawareness \cite{vanDitmarchFrench2009a}, alternative operators
informed through cognitive science \cite{Pietarinen2002}, and dynamic
extensions \cite{Velazquez-Quesada2010-BENTDO-8,grossi2015,Hill2010,vanDitmarchFrench2009a}.

HMS \cite{HMS2008} follow instead Modica and Rustichini \cite{ModicaRustichini1994,ModicaRustichini1999}
and take explicit knowledge as primitive and awareness as defined:
an agent is aware of $\varphi$ iff she either explicitly knows $\varphi$,
or explicitly knows that she does not explicitly know~$\varphi$.

\vspace{-1pt}
\begin{definition}
Let $Ag$ be a finite, non-empty set of agents and $At$ a countable,
non-empty set of atoms. With $a\in Ag$ and $p\in At,$ define the
language $\mathcal{L}$ by
\[
\varphi::=\top\mid p\mid\neg\varphi\mid\varphi\wedge\varphi\mid K_{a}\varphi
\]
and define $A_{a}\varphi:=K_{a}\varphi\vee K_{a}\neg K_{a}\varphi$.

Let $At(\varphi)=\{p\in At\colon p\text{ is a subformula of }\varphi\}$,
for all $\varphi\in\mathcal{L}$.
\end{definition}

\subsection{HMS Models as a Semantics}

The satisfaction of formulas over HMS models is defined as follows.
The semantics are three-valued, so formulas may have undefined truth
value: there may exist a $w\in\Omega_{\mathsf{M}}$ such that neither
$\mathsf{M},w\vDash\varphi$ nor $\mathsf{M},w\vDash\neg\varphi$.
This happens if and only if $\varphi$ contains atoms with undefined
truth value in $w$.
\begin{definition}
Let $\mathsf{M}=(\mathcal{S},\preceq,\mathcal{R},\Pi,V_{\mathsf{M}})$
be an HMS model and let $w\in\Omega_{\mathsf{M}}$. Satisfaction of
\textup{$\mathcal{L}$} formulas is given by
\noindent \begin{center}
\begin{tabular}{lccllllll}
$\mathsf{M},w\vDash\top$ &  & \multicolumn{2}{c}{for all $w\in\Omega_{\mathsf{M}}$} &  &  &  &  & \tabularnewline
$\mathsf{M},w\vDash p$ &  & iff\enskip{} & $w\in V_{\mathsf{M}}(p)$ & \qquad{}\qquad{} & $\mathsf{M},w\vDash\varphi\wedge\psi$ &  & iff\enskip{} & $w\in\llbracket\varphi\rrbracket\cap\llbracket\psi\rrbracket$\tabularnewline
$\mathsf{M},w\vDash\neg\varphi$ &  & iff\enskip{} & $w\in\neg\llbracket\varphi\rrbracket$ &  & $\mathsf{M},w\vDash K_{a}\varphi$ &  & iff\enskip{} & $w\in\boldsymbol{K}_{a}(\llbracket\varphi\rrbracket)$\tabularnewline
\end{tabular}
\par\end{center}

\noindent where $\llbracket\varphi\rrbracket=\{v\in\Omega_{\mathsf{M}}\colon\mathsf{M},v\vDash\varphi\}$
for all $\varphi\in\mathcal{L}$.
\end{definition}

With the HMS semantics being three-valued, they adopt a non-standard
notion of validity which requires only that a formula be always satisfied
\emph{if its has a defined truth value}. The below is equivalent to
the definition in \cite{HMS2008}, but is stated so that it also works
for Kripke lattice models:
\begin{definition}
A formula $\varphi\in\mathcal{L}$ is valid over a class of models
$\boldsymbol{C}$ iff for all models $M\in\boldsymbol{C}$, for all
states $w$ of $M$ which satisfy $p$ or $\neg p$ for all $p\in At(\varphi)$,
$w$ also satisfies $\varphi$.
\end{definition}

\subsection{Kripke Lattice Models as a Semantics}

We define semantics for $\mathcal{L}$ over Kripke lattice models.
Like the HMS semantics, the semantics are three-valued, as it is possible
that a pointed Kripke lattice model $(M,w_{X})$ satisfies neither
$\varphi$ nor $\neg\varphi$. This happens exactly when $\varphi$
contains atoms not in $X$.
\begin{definition}
Let\textbf{ }$\mathsf{K}=(\mathcal{K}(K=(W,R,V)),\trianglelefteqslant,\pi)$
be a Kripke lattice model with $w_{X}\in\Omega_{\mathsf{K}}$. Satisfaction
of \textup{$\mathcal{L}$} formulas is given by

\noindent %
\begin{tabular}{lcc>{\raggedright}p{0.44\textwidth}>{\raggedright}p{0.24\textwidth}}
$\mathsf{K},w_{X}\Vdash\top$ &  &  & for all $w_{X}\in\Omega_{\mathsf{K}}$ & \tabularnewline
$\mathsf{K},w_{X}\Vdash p$  &  & iff\enskip{} & $w_{X}\in V_{X}(p)$ & and $p\in X$\tabularnewline
$\mathsf{K},w_{X}\Vdash\neg\varphi$  &  & iff\enskip{} & not $\mathsf{K},w_{X}\Vdash\varphi$ & and $At(\varphi)\subseteq X$\tabularnewline
$\mathsf{K},w_{X}\Vdash\varphi\wedge\psi$  &  & iff\enskip{} & $\mathsf{K},w_{X}\Vdash\varphi$ and $\mathsf{K},w_{X}\Vdash\psi$\quad{} & and $At(\varphi\wedge\psi)\subseteq X$\tabularnewline
$\mathsf{K},w_{X}\Vdash K_{a}\varphi$ &  & iff\enskip{} & $\pi_{a}(w_{X})R_{Ya}v_{Y}$ implies $\mathsf{K},v_{Y}\Vdash\varphi$, 

for $Y\subseteq At$ s.t. $\pi_{a}(w_{X})\in W_{Y}$ & $\phantom{.}$

and $At(\varphi)\in X$\tabularnewline
\end{tabular}
\end{definition}

\subsection{The Equivalence of HMS and Kripke Lattice Models}

$L$- and $H$-transforms not only produce models of the correct class,
but also preserve finer details, as any model and its transform satisfy
the same formulas.
\begin{proposition}
\label{prop:equivalence}For any HMS model $\mathsf{M}$ with $L$-transform
$L(\mathsf{M})$, for all $\varphi\in\mathcal{L}$, for all $w\in\Omega_{\mathsf{M}}$,
and for all $v\in\ell(w)$, $\mathsf{M},w\vDash\varphi$ iff $L(\mathsf{M}),v\Vdash\varphi$.
\end{proposition}

\begin{proof}
Let $\Sigma_{\mathsf{M}}$ be the events of $\mathsf{M}=(\mathcal{S},\preceq,\mathcal{R},\Pi,V_{\mathsf{M}})$
with maximal state-space $T$ and let $L(\mathsf{M})=(\mathcal{K}(\mathtt{K}=(W,R,V)),\trianglelefteqslant,\pi)$.
The proof is by induction on formula complexity. Let $\varphi\in\mathcal{L}$
and let $w\in\Omega_{\mathsf{M}}$ with $At(S(w))=X$.

\emph{Base:} $i)$ $\varphi:=p\in At$ or $ii)$ $\varphi:=\top$.
$i)$ $\mathsf{M},w\vDash p$ iff $w\in V_{\mathsf{M}}(p)$. As $V_{\mathsf{M}}(p)\in\Sigma_{\mathsf{M}}$,
$(r_{S(w)}^{T})^{-1}(w)\subseteq V_{\mathsf{M}}(p)$. By def. of $L(\mathsf{M})$,
if $v\in T=W$, then $v\in V_{\mathsf{M}}(p)$ iff $v\in V(p)$, so
$v\in(r_{S(w)}^{T})^{-1}(w)$ iff $v\in V(p)$ iff $v_{X}\in V_{X}(p)$,
with $p\in X$ (def. of Kripke lattice models). Hence, by def. of
$\ell$, $v\in\ell(w)=\{u_{X}\in W_{X}\colon u\in(r_{S(w)}^{T})^{-1}(w)\text{ for }X=At(S(w))\}$
iff $v\in V_{X}(p)$, i.e., iff $L(M),v\Vdash p$ for all $v\in\ell(w)$.
$ii)$ is trivial.

\emph{Step.} Assume $\psi,\chi\in\mathcal{L}$ satisfy Prop. \ref{prop:equivalence}.

$\varphi:=\neg\psi$. There are two cases: $i)$ $At(\psi)\subseteq At(S(w))$
or $ii)$ $At(\psi)\not\subseteq At(S(w))$. $i)$ $\mathsf{M},w\vDash\neg\psi$
iff (def. of $\vDash$) $w\in\neg\llbracket\psi\rrbracket$ iff (def.
of $V_{\mathsf{M}}$) $(r_{S(w)}^{T})^{-1}(w)\subseteq\neg\llbracket\psi\rrbracket$
iff (def. of $\llbracket\psi\rrbracket$) for all $v\in(r_{S(w)}^{T})^{-1}(w)$,
$\mathsf{M},v\not\vDash\psi$ iff (Def. \ref{def:L-transform}) for
all $v\in(r_{S(w)}^{T})^{-1}(w)$, not $L(\mathsf{M}),v\Vdash\psi$
iff (def. of $\ell(w)$) for all $v_{X}\in\ell(w)$, not $L(\mathsf{M}),v_{X}\Vdash\psi$,
with $At(\psi)\subseteq X$ iff (def. of $\Vdash$) for all $v_{X}\in\ell(w)$,
$L(\mathsf{M}),v_{X}\Vdash\neg\psi$. $ii)$ is trivial: $\varphi$
is undefined in $(\mathsf{M},w)$ iff it is so in $(L(\mathsf{M}),w_{X})$.

$\varphi:=\psi\wedge\chi$. The case follows by tracing \emph{iff}s
through the definitions of $\vDash$, $V_{\mathsf{M}}$, $\llbracket\cdot\rrbracket$,
$\big(r_{S(w)}^{T}\big)^{-1}$, $L$-transform, $\ell$, and $\Vdash$.

$\varphi:=K_{a}\psi$. $\mathsf{M},w\vDash K_{a}\psi$ iff (def. of
$\vDash$) $w\in\boldsymbol{K}_{a}(\llbracket\psi\rrbracket)$ iff
(def. of $\boldsymbol{K}_{a}$) $\Pi_{a}(w)\subseteq\llbracket\psi\rrbracket$.
Let $\Pi_{a}(w)\subseteq S$, for some $S\in\mathcal{S}$, and let
$X=At(S(w))$ and $Y=At(S)$. Then $v_{S}\in\Pi_{a}(w)\subseteq\llbracket\psi\rrbracket$
iff (def. of $\llbracket\psi\rrbracket$) for all $v_{S}\in\Pi_{a}(w)$,
$\mathsf{M},v_{S}\vDash\psi$ iff (def. of $V_{\mathsf{M}}$) for
all $(r_{S}^{T})^{-1}(v_{S})$ with $v_{S}\in\Pi_{a}(w)$, $\mathsf{M},v_{T}\vDash\psi$
iff (def. of $L$-transform) for all $v_{At}$ with $r_{S}^{T}(v)\in\Pi_{a}(w)$,
$L(\mathsf{M}),v_{At}\Vdash\psi$ and $At(\psi)\subseteq At$ iff
(def. of $L$-transform) for all $v_{At}$ with $(w_{At},v_{At})\in R_{Ata}$,
$L(\mathsf{M}),v_{At}\Vdash\psi$ and $At(\psi)\subseteq At$ iff
(def. of restriction lattice) for all $v_{Y}$ with $(w_{Y},v_{Y})\in R_{Ya}$,
$L(\mathsf{M}),v_{Y}\Vdash\psi$ and $At(\psi)\subseteq Y$ iff (def.
of $\pi_{a}$ and $\pi_{a}(w_{X})=w_{Y}$), for all $v_{Y}$ with
$(\pi_{a}(w_{X}),v_{Y})\in R_{Ya}$, $L(\mathsf{M}),v_{Y}\Vdash\psi$
and $At(\psi)\subseteq Y$ iff (def. of $\Vdash$) $L(\mathsf{M}),w_{X}\Vdash K_{a}\psi$
and $At(\psi)\subseteq Y$.
\end{proof}

\begin{proposition}
\label{prop:equivalence-2}For any Kripke lattice model $\mathsf{K}$
with $H$-transform $H(\mathsf{K})$, for all $\varphi\in\mathcal{L}$,
for all $w_{X}\in\Omega_{\mathsf{K}}$, $\mathsf{K},w_{X}\Vdash\varphi$
iff $H(\mathsf{\mathsf{K}}),w_{X}\vDash\varphi$.
\end{proposition}

\begin{proof}
Let $\mathsf{K}=(\mathcal{K}(K=(W,R,V)),\trianglelefteqslant,\pi)$
with $w_{X}\in\Omega_{\mathsf{K}}$, $\pi_{a}(w_{X})\in W_{Y}$ with
$Y\subseteq At$, and let $H(\mathsf{K})=(\mathcal{S},\preceq,\mathcal{R},\Pi,V_{H(\mathsf{K})})$.
Let $\varphi\in\mathcal{L}$ and proceed by induction on formula complexity.

\emph{Base:} $i)$ $\varphi:=p\in At$ or $ii)$ $\varphi:=\top$.
$i)$ $\mathsf{\mathsf{K}},w_{X}\Vdash p$ iff (def. of $\Vdash$)
$w_{X}\in V_{X}(p)$ with $p\in X$ iff (def. of $H$-transform) $w_{X}\in V_{H(\mathsf{K})}(p)$
iff (def. of $\vDash$) $H(\mathsf{K}),w_{X}\vDash p$. $ii)$ is
trivial.

\emph{Step.} Assume $\psi,\chi\in\mathcal{L}$ satisfy Prop. \ref{prop:equivalence-2}.

$\varphi:=\neg\psi$. There are two cases: $i)$ $At(\psi)\subseteq X$
or $ii)$ $At(\psi)\not\subseteq X$. $i)$ $\mathsf{\mathsf{K}},w_{X}\Vdash\neg\psi$
iff (def. of $\Vdash$) not $\mathsf{K},w_{X}\Vdash\psi$ iff (def.
of $\llbracket\psi\rrbracket$) $w_{X}\not\in\llbracket\psi\rrbracket$
iff (def. of $\llbracket\psi\rrbracket$ and $At(\psi)\subseteq X$)
$w_{X}\in\neg\llbracket\psi\rrbracket$ iff (def. of $\vDash$) $H(\mathsf{K}),w_{X}\vDash\neg\psi$.
$ii)$ is trivial: $\varphi$ is undefined in $(\mathsf{K},w_{X})$
iff it is so in $(H(\mathsf{M}),w_{X})$.

$\varphi:=\psi\wedge\chi$. The case follows by tracing \emph{iff}s
through the definitions of $\Vdash$, $H$-transform, and $\Vdash$.

$\varphi:=K_{a}\psi$. $\mathsf{\mathsf{K}},w_{X}\Vdash K_{a}\psi$
iff (def. of $\Vdash$) $\pi_{a}(w_{X})R_{Ya}v_{Y}$ implies $\mathsf{K},v_{Y}\Vdash\varphi$
iff (def. of $\pi_{a}$, i.e. $\pi_{a}(w_{X})=w_{Y}$ and def. of
$I_{a}$), for all $v_{Y}$ s.t. $(w_{Y},v_{Y})\in R_{Ya}$, i.e.
for all $v_{Y}\in I_{a}(w_{Y})$, $\mathsf{K},v_{Y}\Vdash\varphi$
iff (def. of $\Pi_{a}$, i.e. $\Pi_{a}(w_{X})=I_{a}(\pi_{a}(w_{X})=I_{a}(w_{Y})$)
$\Pi_{a}(w_{X})\subseteq\llbracket\psi\rrbracket$ iff (def. of $\boldsymbol{K}_{a}$)
$w\in\boldsymbol{K}_{a}(\llbracket\psi\rrbracket)$ iff (def. of $\vDash$)
$H(\mathsf{K}),w_{X}\vDash K_{a}\psi$.
\end{proof}

\section{\label{sec:Logic}The HMS Logic of Kripke Lattice Models with Equivalence
Relations}

As we may transition back-and-forth between HMS models and Kripke
lattice models with equivalence relations in a manner that preserve
satisfaction of formula of $\mathcal{L}$, soundness and completeness
of a $\mathcal{L}$-logic is also transferable between the model classes.
We thereby show such results for Kripke lattice models with equivalence
relations as a corollary to results by HMS \cite{HMS2008}.
\begin{definition}
The logic $\Lambda_{HMS}$ is the smallest set of $\mathcal{L}$ formulas
that contain the axioms in, and is closed under the inference rules
of, Table \ref{tab:logic}.
\end{definition}

\begin{table}
\vspace{-24pt}%
\begin{tabular}{|>{\raggedright}p{0.98\textwidth}|}
\hline 
{\small{}All substitution instances of propositional logic, including
the formula $\top$}{\small\par}

{\small{}$A_{a}\neg\varphi\leftrightarrow A_{a}\varphi$\hfill{}(Symmetry)}{\small\par}

{\small{}$A_{a}(\varphi\wedge\psi)\leftrightarrow A_{a}\varphi\wedge A_{a}\psi$\hfill{}(Awareness
Conjunction)}{\small\par}

{\small{}$A_{a}\varphi\leftrightarrow A_{a}K_{b}\varphi$, for all
$b\in Ag$\hfill{} (Awareness Knowledge Reflection)}{\small\par}

{\small{}$K_{a}\varphi\rightarrow\varphi$\hfill{} (T, Axiom of Truth)}{\small\par}

{\small{}$K_{a}\varphi\rightarrow K_{a}K_{a}\varphi$\hfill{} (4,
Positive Introspection Axiom)}{\small\par}

{\small{}From $\varphi$ and $\varphi\rightarrow\psi$, infer $\psi$
\hfill{}(Modus Ponens)}{\small\par}

{\small{}For $\varphi_{1},\varphi_{2},...,\varphi_{n},\varphi$ that
satisfy $At(\varphi)\subseteq\bigcup_{i=1}^{n}At(\varphi_{i})$,}{\small\par}

{\small{}from $\bigwedge_{i=1}^{n}\varphi_{i}\rightarrow\varphi$,
infer $\bigwedge_{i=1}^{n}K_{a}\varphi_{i}\rightarrow K_{a}\varphi$
\hfill{}(RK-Inference)}\tabularnewline
\hline 
\end{tabular}\vspace{2pt}

\caption{\label{tab:logic}Axioms and inference rules of the HMS logic of unawareness,
$\Lambda_{HMS}$.}
\vspace{-24pt}
\end{table}
As the the $L$-transform of an HMS model has equivalence relations,
one may be surprised by the lack of the standard negative introspection
axiom $5:(\neg K_{a}\varphi\rightarrow K_{a}\neg K_{a}\varphi)$ among
the axioms of $\Lambda_{HMS}$. However, including 5 would make collapse
awareness \cite{ModicaRustichini1994}. In \cite{HMS2008}, HMS remarks
that $\Lambda_{HMS}$ imply the weakened version $K_{a}\neg K_{a}\neg K_{a}\varphi\rightarrow(K_{a}\varphi\vee K_{a}\neg K_{a}\varphi)$,
which by the Modici-Rustichini definition of awareness is $K_{a}\neg K_{a}\neg K_{a}\varphi\rightarrow A_{a}\varphi$.
Defining unawareness by $U_{a}\varphi:=\neg A_{a}\varphi$, this again
equates $U_{a}\varphi\rightarrow\neg K_{a}\neg K_{a}\neg K_{a}\varphi$.
Additionally, HMS notes that if $\varphi$ is a theorem, then $A_{a}\varphi\rightarrow K_{a}\varphi$
is a theorem, that $4$ implies introspection of awareness ($A_{a}\varphi\rightarrow K_{a}A_{a}\varphi$),
while $\Lambda_{HMS}$ entails that \emph{awareness is generated by
primitives propositions}, i.e., that $A_{a}\varphi\leftrightarrow\bigwedge_{p\in At(\varphi)}A_{a}p$
is a theorem. The latter two properties entails that HMS awareness
is \emph{propositionally determined}, in the terminology of \cite{HalpernRego2008}.

Using the above given notion of validity and standard notions of proof,
soundness and strong completeness, HMS \cite{HMS2008} state that,
as standard,
\begin{lemma}
\label{lem:4.12}The logic $\Lambda_{HMS}$ is strongly complete with
respect to a class of structures $\mathfrak{S}$ iff every set of
$\Lambda_{HMS}$ consistent formulas is satisfied in some $\mathfrak{s}\in\mathfrak{S}$.
\end{lemma}

\noindent Let $\boldsymbol{M}$ be the class of HMS modes. Using a
canonical model, HMS show:
\begin{theorem}
[\hspace{-1pt}\hspace{1pt}\cite{HMS2008}]\label{thm:HMS-completeness}$\Lambda_{HMS}$
is sound and strongly complete with respect to $\boldsymbol{M}$.
\end{theorem}

Let $\boldsymbol{KLM}_{EQ}$ be the class of Kripke lattice models
where all accessibility relations are equivalence relations. As a
corollary to Theorem \ref{thm:HMS-completeness} and our transformation
and equivalence results, we obtain
\begin{theorem}
$\Lambda_{HMS}$ is sound and strongly complete with respect to $\boldsymbol{KLM}_{EQ}$.
\end{theorem}

\begin{proof}
Soundness: The axioms of $\Lambda_{HMS}$ are valid $\boldsymbol{KLM}_{EQ}$.
We show the contrapositive. Let $\varphi\in\mathcal{L}$. If $\varphi$
is not valid in $\boldsymbol{KLM}_{EQ}$, then for some $\mathsf{K}\in\boldsymbol{KLM}_{EQ}$
and some $w$ from $\mathsf{K}$, $\mathsf{K},w\Vdash\neg\varphi$.
Then its $\mathsf{H}$-transform $H(\mathsf{K})$ is an HMS model
cf. Prop. \ref{prop:H-transform-is-HMS}, and $H(\mathsf{K}),w\vDash\neg\varphi$
cf. Prop. \ref{prop:equivalence-2}. Hence $\varphi$ is not valid
in the class of HMS models. The same reasoning implies that the $\Lambda_{HMS}$
inference rules preserve validity.

Completeness: Assume $\Phi\subseteq\mathcal{L}$ is a consistent set,
and let $\mathfrak{M}$ be the canonical model of HMS, with $\mathfrak{w}$
a state in $\mathfrak{M}$ that satisfies $\Phi$. This exists, cf.
\cite{HMS2008}. By Prop.s \ref{prop:L-transform-is-KLM} and \ref{prop:L-transform-has-EQ-R},
$L(\mathfrak{M})$ is in $\boldsymbol{KLM}_{EQ}$. By Prop. \ref{prop:equivalence},
for all $v\in\ell(\mathfrak{w})$, $L(\mathfrak{M}),v\Vdash\Phi$.
By Lemma \ref{lem:4.12}, $\Lambda_{HMS}$ is thus strongly complete
w.r.t. $\boldsymbol{KLM}_{EQ}$.
\end{proof}

\section{\label{sec:Concluding-Remarks}Concluding Remarks}

This paper has presented a Kripke model-based rendition of the HMS
model of awareness, and shown the two model classes equally general
w.r.t. $\mathcal{L}$, by defining transformations between the two
that preserve formula satisfaction. A corollary to this result is
completeness of the HMS logic for the introduced model class.

There are several issues we would like to study in future work:

In recasting the HMS model, we teased apart the epistemic and awareness
dimensions merged in the HMS possibility correspondences, and Propositions
\ref{prop:L-transform-is-KLM}, \ref{prop:L-transform-has-EQ-R} and
\ref{prop:H-transform-is-HMS} about $L$- and $H$-transforms show
that the HMS properties are satisfied iff each $\pi_{a}$ satisfies
\ref{our1-DownwardsProjection}, \ref{our2-Intro.Idem} and \ref{our3-NoSurpises},
and each $R_{a}$ is an equivalence relation. For a more fine-grained
property correspondence, the propositions' proofs show that each property
of one model is entailed by a strict subset of the properties of the
other. In some cases, the picture emerging is fairly clear: e.g.,
HMS' \ref{HMS1:confinement} is shown only using the restrictions
lattice construction (RLC) plus \ref{our1-DownwardsProjection} and
\emph{vice versa}; \ref{HMS5:PPK} uses only \ref{our3-NoSurpises}
and RLC, while \ref{HMS5:PPK} and \ref{HMS1:confinement} entail
\ref{our3-NoSurpises}. In other cases, the picture is more murky,
e.g., when we use \ref{HMS3:stationarity}, \ref{HMS4:PPI} and \ref{HMS5:PPK}
to show the seemingly simple symmetry of $R_{a}$. We think it would
be interesting to decompose properties on both sides to see if clearer
relationships arise.

There are two issues with redundant states in Kripke lattice models.
One concerns redundant restrictions, cf. Remark \ref{rem:KLM-redundant-states},
which may be solved by working with a more general model class, where
models may also be based on sub-orders of the restriction lattice.
A second one concerns redundant states. For example, in Figure~\ref{fig:Kripke-lattice},
$\mathtt{K}_{\emptyset}$ contains three `identical' states where
no atoms have defined truth values\textemdash $\mathtt{K}_{\emptyset}$
is bisimilar to a one-state Kripke model. As bisimulation contracting
each $\mathtt{K}_{X}$ may collapse states from which awareness maps
differ, one must define a notion of bisimulation that takes awareness
maps into consideration (notions of bisimulation for other awareness
models exists, e.g. \cite{vanDitmarsch2013}). Together with a more
general modal class definition, this could hopefully solve the redundancy
issues.

Though \cite{HMS2008,Schipper2014,HalpernRego2008} provide comparisons
of the HMS and LGA \cite{HalpernFagin88,HalpernRego2008} models,
we would like to make a direct comparison with the latter to understand
Kripke lattice models from an awareness function perspective. It would
then be natural to use the LGA language with awareness and implicit
knowledge as primitives over Kripke lattice models, which is possible
as they include objective states.

The HMS logic is complete for HMS models and Kripke lattice models
with equivalence relations. \cite{HalpernRego2008} prove completeness
for HMS models using a standard validity notion, a `$\varphi$ is
at least as expressive as $\psi$' operator and variants of axioms
$T$, $4$ and $5$. We are very interested in considering this system
and its weaker variants for Kripke lattice models, also with less~assumptions~on~the~relations.

Finally, issues of dynamics spring forth: first, whether existing
awareness dynamics may be understood on Kripke lattice models; second,
whether DEL action models may be applied lattice-wide with reasonable
results; and third, whether the $\pi_{a}$ maps may be thought in
dynamic terms, as they map \mbox{between models.}

\subsubsection*{Acknowledgments.}

We thank the reviewers for their keen eyes and productive comments.
The Center for Information and Bubble Studies is funded by the Carlsberg
Foundation. RKR was partially supported by the DFG-ANR joint project
\emph{Collective Attitude Formation} {[}RO 4548/8-1{]}.

\end{document}